\documentclass{article}

\usepackage[preprint]{neurips_2026}

\usepackage[utf8]{inputenc}
\usepackage[T1]{fontenc}
\usepackage[hidelinks]{hyperref}
\usepackage{url}
\usepackage{booktabs}
\usepackage{amsmath,amssymb}
\usepackage{amsthm}
\usepackage{mathtools}

\newtheorem{proposition}{Proposition}
\newtheorem{lemma}{Lemma}
\newtheorem{assumption}{Assumption}
\newtheorem{definition}{Definition}
\usepackage{microtype}
\usepackage{xcolor}
\usepackage{graphicx}
\usepackage{enumitem}
\usepackage{tikz}
\usetikzlibrary{positioning, arrows.meta, calc}

\title{Completion vs Optimality:\\
Policy Gradient in Long-Horizon\\Cumulative-Damage Problems}

\author{%
  Wolfgang Maass\\
  Saarland University \& German Research Center for Artificial Intelligence (DFKI)\\
  \texttt{wolfgang.maass@iss.uni-saarland.de}
  \And
  Sabine Janzen\\
  German Research Center for Artificial Intelligence (DFKI)\\
  \texttt{sabine.janzen@dfki.de}
}

\begin{document}
\maketitle

\begin{abstract}
Long-horizon decision problems with cumulative damage couple
locally attractive actions to globally adverse outcomes. We
identify two orthogonal failure modes for policy-gradient methods on
this class and propose a decomposition that separates them:
\emph{completion} (reaching the terminal horizon rather than exiting
via an implicit terminal constraint) and \emph{optimality} (matching
the dynamic-programming reference given completion). Under PPO with a
linear soft penalty, granting horizon access alone reduces the
completion rate: the penalty's equilibrium drives the dominant-activity
share to zero, while action-space restriction combined with horizon
access achieves completion but leaves an optimality gap
($\Delta M_{\text{final}} = 0.271$) that we trace to first-phase
greedy commitment at the damage origin. We derive four testable
predictions and evaluate them in two separately calibrated
environments that share the same abstract structure but differ in
domain, horizon, activity set, and calibration data: a 49-step
bricklayer career and a 20-season NBA power-forward career.
All four predictions replicate qualitatively. The horizon-invariance
prediction is met at three of four tested horizons, with the
exception at $H = 15$ consistent with the $H^*$ boundary
($H^* \in [6, 14]$ under the NBA parameters).
\end{abstract}

\section{Introduction}\label{sec:intro}

Many sequential decision problems accumulate state over long horizons
in a way that couples locally attractive actions to globally adverse
outcomes. A clinician dosing a chronic medication, an
engineer scheduling heavy-duty use of a wearing component, and a
worker allocating physical effort across a decades-long career all
face the same structural challenge: the action that maximises
immediate reward increments a cumulative state whose eventual level
terminates the episode. Whether a policy-gradient agent
\citep{schulman2017ppo} can learn the prudent policy in such problems
is a question about credit assignment over horizons that exceed any
rollout-based learner's experience. Recent long-horizon
credit-assignment work \citep{arjonamedina2019,harutyunyan2019} has
identified the structural difficulty but not the specific failure
mode that greedy credit assignment induces when the damage signal is
implicit in the terminal condition.

\paragraph{The diagnostic gap.}
A policy with low return on a long-horizon cumulative-damage problem
may be failing for either of two unrelated reasons. A policy that
terminates early through the implicit terminal constraint has low
return because it never reaches late-horizon states. A policy that
reaches the terminal horizon but commits to locally greedy actions
has low return because it fails to preserve the cumulative state.
Each failure mode calls for a different intervention: horizon access
addresses the former, while credit assignment targets the latter.
A return-based evaluation, however, conflates both onto a single
scalar, forcing a choice between diagnosing \emph{why} the policy
exits early and diagnosing \emph{how suboptimal} it is given that
the episode completes, even when both failures are present
simultaneously and require separate remedies.

\paragraph{This paper.}
We evaluate policies on two independent axes.
\emph{Completion} is whether the policy reaches the terminal horizon.
\emph{Optimality} is whether the policy matches the fixed-share
dynamic-programming reference \citep{puterman1994} given completion.
When the constraint is \emph{implicit} in the terminal condition, no
per-step cost is available, unlike in standard CMDPs
\citep{altman1999,achiam2017,chow2018,garcia2015,ray2019}, and the
standard decomposition does not apply. A principled evaluation
requires making both failure modes separately visible.
For the optimality axis we provide an analytical account: in
Appendix~\ref{sec:synthetic-mdp} we derive, on a minimal
cumulative-damage MDP, a sufficient condition under which the
\emph{expected} first policy-gradient update at the damage origin
points toward the greedy action
(Proposition~\ref{prop:step-zero}). The proposition governs a
binary-action skeleton. Its connection to the multi-activity testbeds
is structural analogy and motivating context, not formal derivation.
We evaluate in two environments (a 49-step bricklayer career and a
20-season NBA power-forward career) sharing the abstract structure
of \S\ref{sec:problem} but differing in activity set, horizon, load
model, and calibration data. In both, PPO on the real environment
fails completion. A Dyna \citep{sutton1990dyna} variant with horizon
access and restricted action space achieves completion but fails
optimality; backward induction achieves both.

\paragraph{Contributions.}
\begin{itemize}[leftmargin=1.5em]
\item A formalisation of the horizon-mismatched cumulative-damage
problem class (\S\ref{sec:problem}), covering latent damage state,
delayed proxy signal, dominant-activity structure, and implicit
role-viability constraint, together with the completion-vs-optimality
decomposition (\S\ref{sec:decomposition}).

\item An empirical three-way comparison (PPO-real, fixed-share
Dyna, fixed-share dynamic programming) that separately quantifies
completion and optimality gaps in two environments
($\Delta M_{\text{final}} = 0.271$ bricklayer,
$0.150$ NBA; 95\% bootstrap CI $[0.148, 0.151]$).

\item A causal analysis showing that, under PPO with a linear soft
penalty, horizon access alone is harmful
(mean exit age $24.7$ vs $27.8$, $p = 0.028$) and action-space
restriction is the intervention that unlocks completion (\S\ref{sec:completion}).

\item A basin-of-attraction analysis identifying first-phase greedy
commitment as the optimality-failure mechanism, with basin entry
within the first 1\% of training regardless of horizon length
(\S\ref{sec:optimality}).
\end{itemize}

\section{Related Work}\label{sec:related}

\paragraph{Constrained MDPs and safe RL.}
The constrained-MDP formulation \citep{altman1999} and its variants
\citep{achiam2017,chow2018} treat the constraint as an explicit,
per-step observed quantity. Safe-exploration methods
\citep{garcia2015,ray2019} treat safety violations as continuous
costs throughout the rollout. Our setting differs structurally: the
terminal condition functions as the constraint, discovered from
rollouts rather than given as a running cost, so standard CMDP
methods cannot diagnose the completion failure we identify.

\paragraph{Policy-gradient convergence.}
A recent theoretical line \citep{mei2020,agarwal2021,bhandari2024}
characterises non-convex basins in policy space. Our step-0 basin
entry instantiates this regime at the damage-state origin, where
immediate reward dominates the Lipschitz tail of the continuation
value. Proposition~\ref{prop:step-zero} is derived from first
principles on a minimal binary-action MDP. Its connection to the
general convergence literature is structural analogy, not formal
derivation. This literature speaks to our optimality axis but is
silent on completion, typically assuming fixed horizons.

\paragraph{Long-horizon credit assignment, reward shaping, and offline methods.}
Methods that rebind credit to earlier actions
\citep{harutyunyan2019,arjonamedina2019} target a mechanism that can
bind either axis. Our decomposition isolates \emph{which} axis is the
binding constraint. Potential-based reward shaping \citep{ng1999} can
alter the gradient landscape near initialisation. Return-conditioned
methods such as Decision Transformer \citep{chen2021dt} might select
trajectories avoiding early termination but require demonstrations and
inherit distribution-shift problems in the sparse late-horizon regime.
Neither class is evaluated here.

\paragraph{Afferent signal design.}
The proxy signal $S_t$ used in the testbeds is a fixed sensitisation
function calibrated to epidemiological data. An alternative not
pursued here is to learn or evolve the afferent signal end-to-end
rather than fixing it. The proxy irrelevance finding
(\S\ref{sec:completion}) is a diagnostic of what fixed afferents
cannot provide, and whether adaptive afferent architectures close
this gap without action-space restriction is an open question.

\section{Problem Class, Decomposition, and Algorithmic Analysis}\label{sec:problem}

\subsection{Problem class}\label{sec:problem-class}

We study finite-horizon MDPs with cumulative damage, specified by
$(\mathcal{S}, \mathcal{A}, T, r, H, \mathcal{E})$, where
$\mathcal{S}$ contains a \emph{cumulative damage variable}
$D_t \in [0,1]$ ($D_{t+1} = D_t + f(s_t, a_t)$, $f \geq 0$) and a
\emph{secondary preservation variable} $M_t \in [0,1]$,
$M_0 = 1$; reward is locally maximised by actions with
$f(s_t,a_t) > 0$; and $\mathcal{E}$ is a set of terminal conditions
that exit the episode when cumulative damage crosses a
state-dependent threshold or when role-viability is violated.
The key distinction from standard CMDPs
\citep{guin2022,wachi2024} is that the constraint in $\mathcal{E}$
is \emph{never observed as a running cost}: the learner discovers
the terminal boundary only by reaching it. We call such a problem
\emph{horizon-mismatched} when the greedy policy's effective episode
length is strictly less than $H$.

\begin{assumption}[Cumulative-damage structure]
\label{asm:cds}
A cumulative-damage MDP $\mathcal{M}$ satisfies:
(i)~\textbf{Monotone damage}: $f(s,a) \geq 0$ for all $(s,a)$, so
$D_t$ is non-decreasing along any trajectory.
(ii)~\textbf{Greedy-damage alignment}: the reward-maximising action
satisfies $f(s,\arg\max_a r(s,a)) > 0$.
(iii)~\textbf{Implicit terminal boundary}: $\mathcal{E}$ is
non-empty and is \emph{avoided} with positive probability only by
policies under which the greedy action is not always selected.
\end{assumption}

Beyond Assumption~\ref{asm:cds}, the instances we study share five
structural features. The first two concern what the agent observes
and the structure of its action space: (a)~\textbf{latent damage state}: $D_t$ and
the secondary preservation variable $M_t$ are unobserved; the agent
receives only a proxy signal $S_t = g(D_t)$ that is
uninformative at the low-damage levels the greedy policy visits, and
therefore no observation-based warning reaches the agent before the implicit
boundary. We use a fixed sensitisation function calibrated to
epidemiological data as a diagnostic baseline. (b)~\textbf{dominant-activity structure}: one activity is
simultaneously the unique reward maximiser and the unique damage
maximiser, with $k-1$ lower-reward, lower-damage alternatives
creating a substitution structure whose unconstrained optimum
conflicts with the role-viability constraint~(c).
The remaining three govern how the environment enforces and penalises
damage: (c)~\textbf{role-viability constraint}: the episode
terminates if the trailing-window dominant-activity share falls below
$\alpha$, operating in the share dimension of the action space;
(d)~\textbf{capacity-damage feedback}: per-step reward is suppressed
by $h(D_t, S_t)$, silent below $D_{\text{clin}}$ and progressive
above it; and (e)~\textbf{self-amplifying secondary channel}: $M_t$
degrades at an accelerating rate below $M_{\text{amp}} = 0.6$,
making early maximum-effort trajectories irreversibly worse than
proactive low-effort ones. Without~(e) the optimality gap
vanishes: DP and the greedy policy produce similar $M_{\text{final}}$
at matched career lengths; both testbeds use $M_{\text{amp}} = 0.6$.

\subsection{The completion-vs-optimality decomposition}\label{sec:decomposition}

\begin{definition}[Completion]
\label{def:completion}
A policy $\pi$ achieves \emph{completion} if the probability of
reaching any element of $\mathcal{E}$ before step $H$ is zero: the
expected episode length equals $H$.
\end{definition}

\begin{definition}[Optimality]
\label{def:optimality}
Given completion, a policy $\pi$ achieves \emph{optimality} if
$V^{\pi}(s_0) = V^{\pi^{\ast}}(s_0)$, where $\pi^{\ast}$ is the
backward-induction reference optimised over the same action subspace
as $\pi$. Optimality is undefined for policies that fail completion.
\end{definition}

These definitions induce a three-outcome ordinal scale
(Table~\ref{tab:framework}). The \emph{optimality gap} is
$\Delta M_{\text{final}} = M^{\text{DP}}_{\text{final}} -
M^{\pi}_{\text{final}}$. Definition~\ref{def:optimality} is stated
in terms of $V^\pi$. $\Delta M_{\text{final}}$ serves as a
domain-specific surrogate for the value gap, justified by
feature~(e): under the self-amplifying secondary channel and the
fixed-share constraint, $M_{\text{final}}$ is a monotone function of
$V^\pi$ given completion, so $V^\pi < V^{\pi^*} \Leftrightarrow
M^\pi_{\text{final}} < M^{\pi^*}_{\text{final}}$. Accordingly,
$\Delta M_{\text{final}} = 0$ corresponds to Cell~C and
$\Delta M_{\text{final}} > 0$ to Cell~B.

\begin{table}[t]
\caption{Three-outcome scale induced by Definitions~\ref{def:completion}
and~\ref{def:optimality}. The scale is strictly ordered
Cell~A $\prec$ Cell~B $\prec$ Cell~C. Optimality is only
meaningful once completion is achieved; Cell~A vs.\ Cell~B/C
comparisons on $M_{\text{final}}$ confound truncation with
suboptimality and are not made in this paper.}
\label{tab:framework}
\centering
\begin{tabular}{clp{5.8cm}}
\toprule
\textbf{Cell} & \textbf{Completion} & \textbf{Optimality} \\
\midrule
A & fails   & undefined (episode truncated before late-horizon states; $M_{\text{final}}$ artificially high) \\
B & achieves & fails ($V^{\pi} < V^{\pi^{\ast}}$, optimality gap $> 0$) \\
C & achieves & achieves ($V^{\pi} = V^{\pi^{\ast}}$, optimality gap $= 0$) \\
\bottomrule
\end{tabular}
\end{table}

\begin{figure}[t]
\centering
\resizebox{\textwidth}{!}{%
\begin{tikzpicture}[
  font=\scriptsize,
  block/.style={draw, rounded corners=1.5pt, minimum width=2.8cm,
                minimum height=0.55cm, align=center, inner sep=2pt},
  train/.style={block, fill=orange!25, draw=orange!75!black, thick},
  frozen/.style={block, fill=cyan!18, draw=cyan!60!black, thick},
  env/.style={block, fill=gray!12, draw=gray!50},
  outcome/.style={block, fill=yellow!28, draw=olive!50!black,
                  font=\scriptsize\bfseries},
  arr/.style={-{Stealth[length=2mm]}, semithick},
  laneheader/.style={font=\bfseries\scriptsize}
]

\def\colA{0}
\def\colB{4.2}
\def\colC{8.4}
\def\vsep{0.18cm}

\node[laneheader] (hA) at (\colA, 0) {(A) PPO-real};
\node[laneheader] (hB) at (\colB, 0) {(B) Fixed-share Dyna};
\node[laneheader] (hC) at (\colC, 0) {(C) DP-optimal};

\node[env, below=\vsep of hA] (eA) {Environment\\\tiny latent $D_t, M_t$};
\node[env, below=\vsep of hB] (eB) {Environment\\\tiny latent $D_t, M_t$};
\node[env, below=\vsep of hC] (eC) {Environment\\\tiny full state $(D,M,t)$};

\node[env, below=\vsep of eA] (obsA) {Observation\\\tiny proxy $S_t{=}g(D_t)$};
\node[env, below=\vsep of eB] (obsB) {Observation\\\tiny proxy $S_t$};
\node[env, below=\vsep of eC] (obsC) {Direct access\\\tiny $(D,M,t)$ on grid};

\node[train,  below=\vsep of obsA] (shA) {Share head\\\tiny learned via PG};
\node[frozen, below=\vsep of obsB] (shB) {Share head\\\tiny \texttt{fixed\_share}: DP-clamped};
\node[frozen, below=\vsep of obsC] (shC) {Share allocation\\\tiny backward induction};

\node[train,  below=\vsep of shA] (efA) {Effort head\\\tiny learned via PG};
\node[train,  below=\vsep of shB] (efB) {Effort head\\\tiny learned via PG};
\node[frozen, below=\vsep of shC] (efC) {Effort policy\\\tiny backward induction};

\node[env, below=\vsep of efA] (hvA) {Role-viability gate\\\tiny active: terminates early};
\node[env, below=\vsep of efB] (hvB) {Role-viability gate\\\tiny \texttt{no\_exit}: bypassed};
\node[env, below=\vsep of efC] (hvC) {Backward pass\\\tiny full horizon $H$};

\node[outcome, below=\vsep of hvA] (oA) {Cell A\\\tiny fails completion ($27.8$ yr)};
\node[outcome, below=\vsep of hvB] (oB) {Cell B\\\tiny completion; $\Delta M_{\text{final}}{=}0.271$};
\node[outcome, below=\vsep of hvC] (oC) {Cell C\\\tiny completion $+$ optimality};

\foreach \a/\b in {eA/obsA, obsA/shA, shA/efA, efA/hvA, hvA/oA,
                   eB/obsB, obsB/shB, shB/efB, efB/hvB, hvB/oB,
                   eC/obsC, obsC/shC, shC/efC, efC/hvC, hvC/oC}
  \draw[arr] (\a) -- (\b);

\node[train,   anchor=north west, minimum width=1.9cm, minimum height=0.45cm]
  (lg1) at ($(oA.south west) + (-0.3,-0.35)$) {\tiny Trainable (PG)};
\node[frozen,  right=0.15cm of lg1, minimum width=1.9cm, minimum height=0.45cm]
  (lg2) {\tiny Frozen / clamped};
\node[env,     right=0.15cm of lg2, minimum width=2.0cm, minimum height=0.45cm]
  (lg3) {\tiny Environment / fixed};
\node[outcome, right=0.15cm of lg3, minimum width=1.5cm, minimum height=0.45cm]
  (lg4) {\tiny Outcome};

\end{tikzpicture}
}%
\caption{Three-way method-comparison architecture. Columns share the
pipeline state $\to$ observation $\to$ share/effort heads $\to$
role-viability gate $\to$ outcome and differ only in which components
are trainable (orange) vs.\ clamped (cyan). \texttt{fixed\_share}
replaces the trainable share head; \texttt{no\_exit} bypasses the
role-viability gate.}
\label{fig:method-arch}
\end{figure}
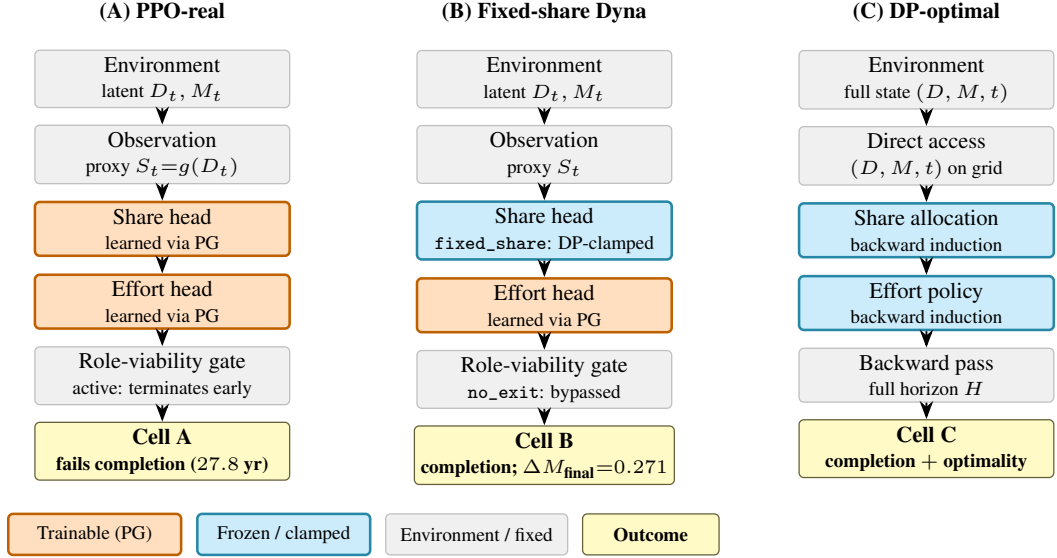

\subsection{Mechanism analysis: backward induction vs.\ policy gradient}\label{sec:dp-vs-pg}

\paragraph{Backward induction achieves both axes.}
DP evaluates every $(D, M, t)$ grid point backward from $H$, and therefore
$\mathcal{E}$ is avoided \emph{a priori} and the capacity-damage
feedback~(d) is integrated over the full remaining
horizon. At $(D_0\!=\!0, M_0\!=\!1, t\!=\!0)$, sustained low effort
is identified as the unique trajectory keeping $D$ below
$D_{\text{clin}}$. Latent state~(a) is moot under direct grid access
(Figure~\ref{fig:method-arch}, column~C).

\paragraph{Policy gradient fails completion.}
Under Assumption~\ref{asm:cds}(ii)--(iii), the greedy policy reaches
$\mathcal{E}$ before $H$. When horizon access is granted via
\texttt{no\_exit}, the linear soft penalty's per-step cost outweighs
the dominant activity's reward advantage, and therefore the share head's
unconstrained optimum places mass at zero on the dominant activity,
violating the constraint the hard exit would have enforced.
Completion failure arises from two interacting causes: the implicit
terminal boundary and the uninformative afferent in the greedy
policy's operating damage range, addressable via action-space
restriction (\texttt{fixed\_share}, studied here) or by improving
the afferent's informativeness in the policy's operating damage
range. Clamping shares to a
constraint-satisfying allocation structurally prevents role-viability
violations regardless of whether horizon access is granted.

\paragraph{Policy gradient fails optimality.}
Even with completion achieved, the dominant-activity structure~(b)
creates a first-phase commitment at $D = 0$: the reward gap
$e_H^\beta - e_L^\beta$ is large while damage cost is zero, and therefore the
first policy-gradient update (in expectation) commits to maximum
effort on the dominant activity.
Proposition~\ref{prop:step-zero} (Appendix~\ref{sec:synthetic-mdp})
formalises this as a sufficient condition on the \emph{expected}
gradient direction in a binary-action minimal MDP. The
self-amplifying channel~(e) then ensures that once $M$ falls
below $M_{\text{amp}}$, the degradation cascade is irreversible.

\subsection{Testable predictions}\label{sec:predictions}

The analysis of \S\ref{sec:dp-vs-pg} yields four testable predictions
about any cumulative-damage MDP satisfying Assumption~\ref{asm:cds}
with features (a)--(e), under PPO with a linear soft penalty for the
completion-axis predictions (P1--P2). The proxy signal~(a)
contributes nothing detectable to completion because its informative
range ($D > D_{\text{clin}}$) is disjoint from the greedy policy's
operating damage range (confirmed empirically in \S\ref{sec:completion}).
Backward induction achieves both axes by construction and occupies
Cell~C; it serves as the upper-bound reference throughout.

\begin{enumerate}[label=\textbf{P\arabic*},leftmargin=2.5em]
\item \label{pred:completion-fail}
\textbf{PG-real fails completion.} Policy gradient trained on real
rollouts with $\mathcal{E}$ enforced exits early, occupying Cell~A.

\item \label{pred:harmful-horizon}
\textbf{Horizon access alone is harmful} (under PPO with linear soft
penalty): granting horizon access without action-space restriction
shortens careers relative to PG-real, because the penalty's
equilibrium under~(c) drives the dominant-activity share below
$\alpha$.

\item \label{pred:reactive-attractor}
\textbf{Reactive-decline attractor.} Under completion and given the
self-amplifying structure~(e), the dominant attractor is maximum
effort at $t = 0$, monotone decline, $M$ irreversibly below
$M_{\text{amp}}$, and positive $\Delta M_{\text{final}}$.

\item \label{pred:horizon-invariance}
\textbf{Horizon invariance.} The reactive attractor's dominance rate
and $\Delta M_{\text{final}}$ are invariant across training horizons
$H \leq H^*$, because basin entry occurs at step~0.
\end{enumerate}

\section{Testbeds and Headline Results}\label{sec:testbeds}

We test P1--P4 in two cumulative-damage environments. Both satisfy
Assumption~\ref{asm:cds} and features (a)--(e) but differ in
activity set, horizon, load model, biological calibration data, and
role-viability parameters. The
two environments share the same simulation engine, the HMS regulation
pathway, and the Baratz amplification function. The replication
therefore tests whether the findings generalise across
parameter regimes of the same problem class, not across
independently developed simulators, a limitation we return to in
\S\ref{sec:discussion}.

The \textbf{bricklayer} testbed models a 49-year career over $k=7$
construction activities, calibrated to a bio-physical model of the
human knee \citep{coggon2000,jensen2008,rytter2009,baratz1986,arendtnielsen2010}
(Appendix~\ref{sec:bio-params}). The \textbf{NBA power-forward}
testbed models a 20-season career over 6 activities, calibrated to
the sports-medicine literature \citep{drakos2010,wiggins2016}
(Appendix~\ref{sec:nba-params}).

Figure~\ref{fig:three-way} and Table~\ref{tab:decomposition}
summarise the headline three-way contrast. The DP reference is
computed in under two hours on a single CPU
(Appendix~\ref{sec:hyperparams}), and the gap cannot therefore be
attributed to the cost of computing the DP reference.

\begin{figure}[t]
\centering
\includegraphics[height=7cm,keepaspectratio]{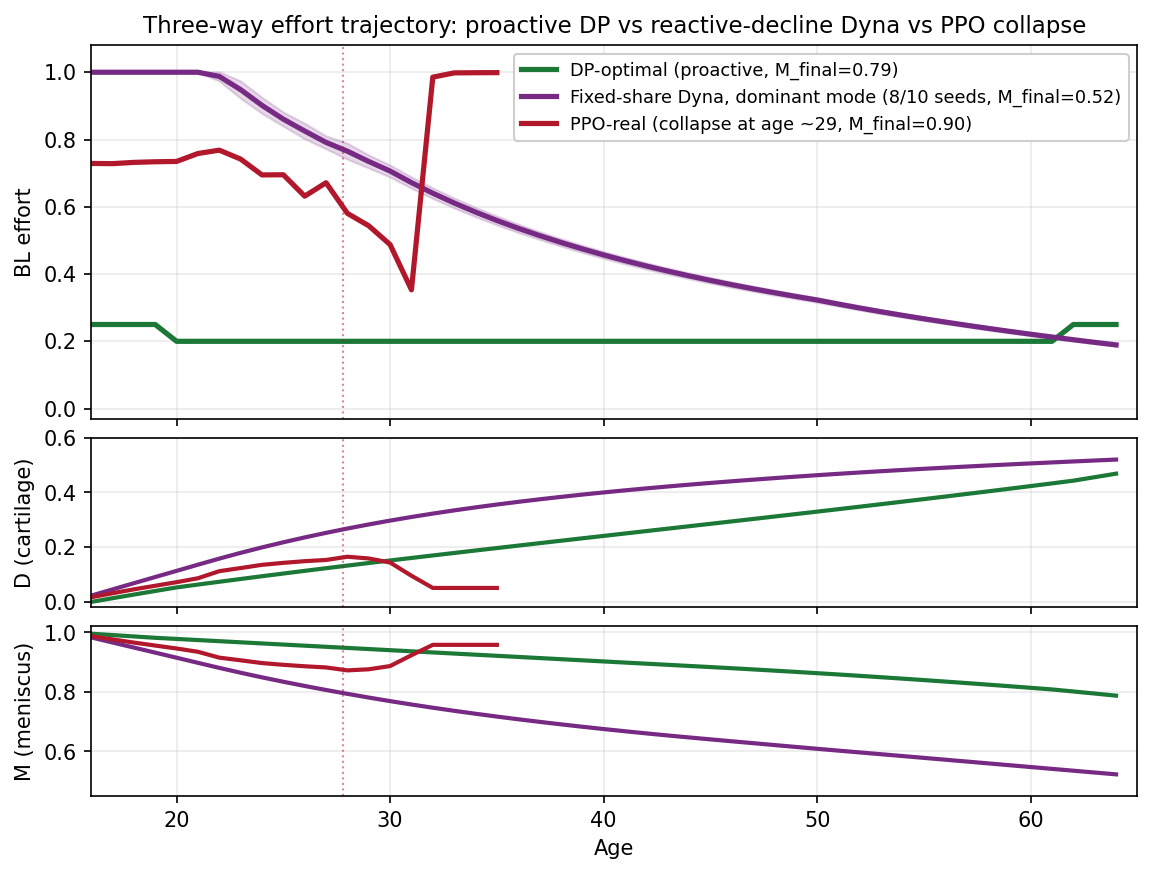}
\caption{Three-way dominant-activity effort (bricklayer, $H=49$).
DP-optimal (green, Cell~C, $M_{\text{final}}{=}0.79$);
dominant-mode Dyna (purple, Cell~B, 8/10 seeds, $0.52$);
PPO-real (red, Cell~A, exit~$\approx 27.8$\,yr).
Headline gap $\Delta M_{\text{final}} = 0.271$.}
\label{fig:three-way}
\end{figure}

\begin{table}[t]
\caption{Bricklayer instantiation of Table~\ref{tab:framework}. Cell~A
$M_{\text{final}}$ omitted (truncation confounds with suboptimality).
$\dagger$~career ends at age~65 ($H=49$ steps).}
\label{tab:decomposition}
\centering
\begin{tabular}{llcccc}
\toprule
\textbf{Condition} & \textbf{Cell} & \textbf{Seeds} &
\textbf{Exit age (mean $\pm$ sd)} & \textbf{Completion} & \textbf{Optimality} \\
\midrule
PPO-real                 & A & 20   & $27.8 \pm 4.0$ & fails             & ---               \\
Unrestricted Dyna        & A & 20   & $24.7 \pm 4.5$ & fails             & ---               \\
Fixed-share Dyna (dom.)  & B & 8/10 & $65.0 \pm 0.0$ & \textbf{achieves} & fails             \\
DP-optimal               & C & --   & $65^\dagger$   & \textbf{achieves} & \textbf{achieves} \\
\bottomrule
\end{tabular}
\end{table}

\section{Completion: Causal Analysis}\label{sec:completion}

\paragraph{PPO-real fails completion (P1).}
PPO with $\gamma = 1$ (undiscounted episodic return; the rbar
exponential moving-average baseline of \citealt{mahadevan1996} is
used for variance reduction)
terminates at a mean age of $27.8 \pm 4.0$ years, with 100\% of
2{,}000 eval episodes ending in role-violation (20 seeds, 10
curriculum-on and 10 curriculum-off). The role-viability rule fires
when
\[
\frac{1}{W}\sum_{\tau = t - W + 1}^{t} s^{(\tau)}_{\text{dom}} < \alpha,
\]
with $W = 5$ years and $\alpha = 0.15$ in the bricklayer.
The agent observes only the proxy
signal $S_t$. The trailing share history
$s^{(\tau)}_{\text{dom}},\, \tau \in [t{-}W{+}1, t{-}1]$ that
determines the role-viability boundary is not provided in the state
observation, making the constraint boundary non-Markovian from the
agent's perspective. This partial-observability challenge is a
contributing factor to PPO-real's failure, separate from and
additional to the long-horizon credit-assignment difficulty. Each
is acknowledged in the Limitations of \S\ref{sec:discussion}.
Including the proxy signal $S$ in the state yields $\Delta = -0.65$\,yr
relative to zeroing it (20 seeds per condition), within seed variance
and in the wrong direction: $S$ becomes informative only at
$D > 0.3$, whereas PPO-real exits at $D_{\text{final}} = 0.16$.
The proxy is not uninformative in principle. It is informative about
states the learner never reaches, a diagnostic of the fixed
afferent's design, which covers the clinical damage range rather than
the early-career range where the greedy policy operates.

\paragraph{Horizon access alone is harmful (P2).}
Across 20 seeds of the best unrestricted Dyna configuration, trained
with rollouts extended past role-violation (\texttt{no\_exit} flag)
and a linear soft role penalty with weight $w = 1.0$ (applied as
$-w \cdot \mathbf{1}[s_{\text{dom},t} < \alpha]$ per step),
the mean exit age under real evaluation is $24.7$ years
(95\% CI $[23.0, 26.8]$), \emph{earlier} than PPO-real's $27.8$
(Welch's $t(37.8) = -2.29$, $p = 0.028$; Mann-Whitney $p = 0.009$).
The exit-age distribution is bimodal: 10/20 seeds exit at
age~21 via role-violation (the earliest age the trailing-window
check can fire); the remaining 10 seeds survive to ages 26--38.
The linear soft penalty does not induce an interior basin around
any positive dominant-activity allocation, and therefore the share
head's unconstrained optimum places mass at zero once curriculum
scaffolding lifts. Whether convex or entropy-regularised
alternatives escape this pattern is open
(\S\ref{sec:discussion}).

\paragraph{Action-space restriction unlocks completion.}
Restricting the policy to the DP-optimal share allocation
(\texttt{fixed\_share}, learning only effort) achieves completion
in all cases: all 10 seeds reach the terminus
($65.0 \pm 0.0$ yr, zero role-violations over 1{,}000 eval episodes)
across $H \in \{13, 20, 30, 49\}$
(Appendix~\ref{appendix:horizon-invariance}). Completion is a
function of action-space structure, not horizon length.

\paragraph{Fixed-share PPO-real ablation.}
To isolate whether action-space restriction alone (without
\texttt{no\_exit}) suffices, we run \texttt{fixed\_share} PPO-real
with the role-viability gate active during training (10 seeds).
All 10 seeds achieve 100\% completion (exit age $65.0 \pm 0.0$).
Completion is therefore attributable to the structural share
constraint, which keeps $s_{\text{dom}} \geq \alpha$ by
construction. \texttt{no\_exit} is not independently necessary
for completion but is required for the Dyna agent to accumulate
long-horizon experience.

A $(W, \alpha)$ sensitivity sweep (90 runs,
Appendix~\ref{sec:role-sensitivity}) confirms threshold robustness.

\section{Optimality: Basin Analysis}\label{sec:optimality}

\subsection{Reactive decline is the dominant attractor (P3)}\label{sec:optimality:reactive}

Across 10 seeds of the fixed-share Dyna sweep, 8 converge to a nearly
identical effort profile: maximum effort at $t = 0$, monotone
decline to $0.22$ by age 60 (Figure~\ref{fig:basin}b).
Seed-to-seed variance in
dominant-mode $M_{\text{final}}$ is $0.001$, three orders of
magnitude tighter than the gap to DP-optimal, confirming this is a
distinct fixed point, not optimisation noise.

The mechanism is feature~(e): early maximum effort drives $M$ below
$M_{\text{amp}} = 0.6$ irreversibly, triggering a degradation
cascade the DP avoids by sustaining low effort from $t = 0$.
Two seeds escape the reactive basin: seed~4 via an inverted-U
effort profile ($M_{\text{final}} = 0.64$) and seed~7 via a
mid-training reward collapse to a DP-like low-effort mode
($M_{\text{final}} = 0.825$), demonstrating that the DP-competitive
region is stochastically reachable, though not via a reliable
convergence path. The basin-dominance claim is
accordingly about 8 of 10 seeds (95\% Clopper--Pearson CI
$[0.44, 0.97]$), not 10 of 10.

\subsection{First-phase basin entry as the mechanism (P4)}\label{sec:optimality:step-zero}

Three pieces of evidence support first-phase basin entry:

\begin{enumerate}
\item \textbf{All 8 dominant-cluster seeds reach maximum effort at
the initial state within the first 10{,}000 training steps}
($\approx$1\% of training), read from periodic training logs. The
reactive trajectory is established before any long-horizon signal has
accumulated.

\item \textbf{The two escaping seeds did not commit at step~0.}
Seed~4 has initial-state effort $0.183$ and seed~7 has $0.346$;
both escaped because their stochastic first gradient update did not
push toward the maximum. Proposition~\ref{prop:step-zero}
governs the \emph{expected} gradient direction. Individual sample
gradients can differ in sign, which is consistent with these two
escaping seeds.

\item \textbf{The basin structure is invariant to training horizon}
across $H \in \{13, 20, 30, 49\}$: dominant-mode
$M_{\text{final}}$ is constant at $0.523 \pm 0.001$ regardless of
horizon (Appendix~\ref{appendix:horizon-invariance},
Figure~\ref{fig:horizon-invariance}). A basin entered by
aggregating long-horizon signals would shift with $H$. This one does
not.
\end{enumerate}

Proposition~\ref{prop:step-zero} (Appendix~\ref{sec:synthetic-mdp})
formalises the \emph{direction of the expected gradient} in a
binary-action minimal MDP. The empirical evidence above establishes
saturation within the first 1\% of training for all 8 reactive-cluster
seeds.

\subsection{Initialization intervention: a falsifiability test}\label{sec:optimality:init-intervention}

The first-phase account predicts that biasing the effort head's
initial output toward the DP-optimal value weakens the basin's
dominance. Setting the effort head's final-layer bias to $-1.386$
($\sigma(-1.386) \approx 0.20$, matching DP-optimal sustained effort),
with all other settings identical to the fixed-share sweep, produces
\textbf{4/10 seeds escaping} (vs.\ 2/10 baseline) and raises mean
$M_{\text{final}}$ from $0.564$ to $0.609$.

Pre-registered binding thresholds (\texttt{prompts/e1\_prereg.md},
frozen 2026-04-20): $\geq 6$ confirms, $4$--$5$ is \emph{nuanced},
$\leq 3$ falsifies. The result falls in the nuanced band. Six seeds
revert to the reactive basin within the first $10^4$ gradient
updates. The four escapers split into three attractor types:
seeds~4 and~7 reach DP-competitive
$M_{\text{final}}{\in}[0.79, 0.85]$; seed~9 converges to near-zero
late-career effort ($e_{t>40} \approx 0$); seed~0 oscillates.
Per-seed results are in
Appendix~\ref{sec:hyperparams}, Table~\ref{tab:e1-seed-table}).
Seeds~4 and~7 represent qualitatively distinct effort profiles:
seed~4 converges to a gradually increasing schedule near the
DP-optimal level, while seed~7 settles below the DP grid minimum
($e < 0.20$) and achieves $M_{\text{final}} = 0.846$, exceeding the
DP reference. Both escape the reactive basin not by resolving the
long-horizon credit-assignment problem but by avoiding step-0
commitment, a pattern consistent with the occupational health finding
that conservative early-career effort profiles produce better
long-run outcomes than maximum early effort \citep{mclellan2022}.
The attraction is strong but not absolute: closing the gap requires
weakening the basin's attraction, not correcting the initialisation.

\begin{figure}[t]
\centering
\includegraphics[height=5cm,keepaspectratio]{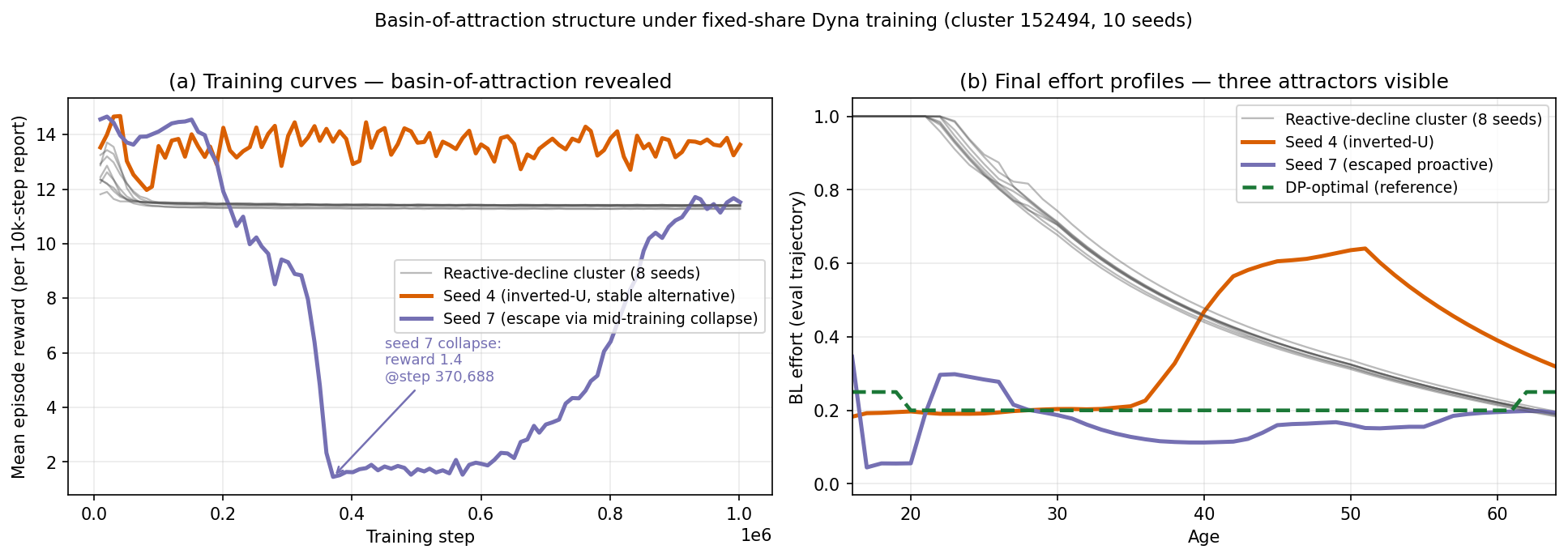}
\caption{Basin-of-attraction structure (10 seeds, bricklayer).
(a)~8 seeds converge to the reactive-decline attractor ($\sim$100k
steps); seed~4 reaches a stable alternative; seed~7 collapses
mid-training and partially recovers, demonstrating that the
DP-competitive region is stochastically reachable.
(b)~Final effort profiles; all 8 reactive seeds output maximum
effort after only 10k steps.}
\label{fig:basin}
\end{figure}

\section{Replication: NBA Career Environment}\label{sec:replication}

To distinguish the bricklayer findings from parameter-specific
artefacts, we replicate the four-condition experiment in the NBA
power-forward testbed (\S\ref{sec:testbeds}).

\paragraph{Completion axis (P1--P2) replicates.}
PPO-real fails across all 20 seeds (exit $22.6 \pm 1.8$,
role-violation $100\%$). Unrestricted Dyna across 60 seeds
($w \in \{0.5, 1.0, 2.0\}$) exits at $21.2$ yr, $1.4$ yr
\emph{earlier} than PPO-real ($p = 0.003$ Welch's $t$;
$p = 0.0003$ Mann-Whitney). The same direction of harm confirms the
mechanism is not domain-specific. Fixed-share Dyna achieves $100\%$
completion (10/10 reach age~38).

\paragraph{Optimality axis (P3--P4) replicates with compression.}
$\Delta M_{\text{final}} = 0.150$ (95\% CI $[0.148, 0.151]$), above
the pre-registered $0.10$ threshold. The NBA testbed falls outside the
analytical guarantee ($H^* \in [6, 14]$,
Appendix~\ref{sec:synthetic-mdp}), providing a direct test of whether
the commitment condition is sufficient but not necessary: the reactive
basin is the plurality attractor ($4/10$) rather than the
bricklayer's majority ($8/10$), consistent with the $H^*$ compression
prediction. Three seeds settle in a mid-cluster
($M_{\text{final}}{=}0.71$) and three converge near DP ($0.85$).
Fisher's exact $p = 0.17$ at $n=10$ is non-significant
per-environment, but the decomposition generalises across domains.

\paragraph{Horizon invariance (P4) replicates at 3 of 4 horizons.}
Dominant-mode $M_{\text{final}}$ is $0.614, 0.614, 0.648, 0.615$
for $H = 7, 10, 15, 20$. Three horizons agree to within $0.001$;
the $H = 15$ outlier is consistent with $H^*$ in $[6, 14]$
(Appendix~\ref{sec:synthetic-mdp}): $H = 15$ lies outside the
analytical guarantee under both damage channels. The basin-dominance
rate decreases as $H$ exceeds $H^*$. The step-0 mechanism remains
universal.

\begin{table}[t]
\caption{Bricklayer ($H=49$) vs.\ NBA ($H=20$) replication summary.}
\label{tab:nba-replication}
\centering\small
\begin{tabular}{lcc}
\toprule
\textbf{Finding} & \textbf{Bricklayer} & \textbf{NBA} \\
\midrule
PPO-real completion rate                  & $0\%$               & $0\%$ \\
PPO-real mean exit age                    & $27.8\pm4.0$        & $22.6\pm1.8$ \\
Unrestr.\ Dyna $\Delta$ exit vs PPO-real  & $-3.1$ yr & $-1.4$ yr \\
Welch's $t$ ($p$) / Mann-Whitney ($p$)    & $0.028$ / $0.009$ & $0.003$ / $0.0003$ \\
Fixed-share Dyna completion rate          & $100\%$ & $100\%$ \\
$M^{\text{DP}}_{\text{final}}$            & $0.794$ & $0.764$ \\
$M^{\text{Dyna,dom}}_{\text{final}}$              & $0.523 \pm 0.001$   & $0.614 \pm 0.002$ \\
$\Delta M_{\text{final}}$ (95\% CI)               & $0.271$             & $0.150\ [0.148,0.151]$ \\
Reactive-basin dominance (95\% CP CI)     & $8/10\ [0.44,0.97]$ & $4/10\ [0.12,0.74]$ \\
Fisher exact ($p$, dominance difference)  & \multicolumn{2}{c}{$p = 0.17$; not significant at $n=10$} \\
\bottomrule
\end{tabular}
\end{table}

\section{Discussion}\label{sec:discussion}

\paragraph{The decomposition as diagnostic.}
The decomposition separates two failure modes that return-based
evaluation conflates, and shows that interventions targeting one
axis can worsen the other: a paper reporting only
exit age would claim fixed-share Dyna is solved, while one reporting
only $M_{\text{final}}$ would be misled by PPO-real's artificially
high (early-exit) score. The decomposition also maps existing
interventions to the axis each addresses. Action-space restriction
(\texttt{fixed\_share}) targets the completion axis by enforcing a
constraint-satisfying share allocation. Improving the afferent's
informativeness in the policy's operating damage range targets the
completion axis from the information side. The optimality gap ($\Delta
M_{\text{final}} = 0.271$) remains open to both interventions and
requires a separate credit-assignment solution. The reactive-decline policy is qualitatively
consistent with occupational epidemiology
\citep{palmer2012,coggon2000,jensen2008,rytter2009}: workers
maintain high effort until tissue damage forces reactive regulation
\citep{woolfsalter2000,arendtnielsen2010}, though real workers face
financial and social constraints we do not model. We conjecture the
failure modes generalise to other cumulative-damage domains
satisfying Assumption~\ref{asm:cds} with features (a)--(e).

\paragraph{Connection to early sport specialisation.}
The reactive-decline attractor has a direct empirical counterpart in
the early sport specialisation literature \citep{difiori2014,mclellan2022,jildeh2024}:
maximum dominant-activity effort from career outset, accelerating
secondary degradation once a tissue threshold is crossed, and
premature exit. The structural parallels are precise: latent damage
accumulation before pain signals become informative (feature~(a)),
self-amplifying degradation below $M_{\text{amp}}$ (feature~(e)),
and implicit role-viability violation as the terminal condition.
Empirically, NBA players who specialised in a single sport during
adolescence play in fewer career games and suffer more serious
injuries than multi-sport athletes \citep{mclellan2022}, consistent
with the model's prediction that distributing effort across activities
preserves $M_{\text{final}}$. The step-0 commitment mechanism
provides a formal account of why agents in this class converge to
this pattern even when proactive load management would be optimal.

\paragraph{Limitations.}
\textit{Same engine}: replication tests parameter-regime generality within one codebase, not across independently developed simulators.
\textit{Proposition--testbed gap}: Proposition~\ref{prop:step-zero} covers a binary-action minimal MDP; it motivates but does not formally derive the empirical findings in the multi-activity continuous-effort testbeds.
\textit{PPO and penalty specificity}: whether convex penalties or maximum-entropy methods (SAC, \citealp{haarnoja2018}) escape harmful horizon access is open.
\textit{Ground-truth model}: Dyna uses the exact simulator transition; practical model-based RL incurs additional misspecification error.
\textit{Partial observability}: the role-viability boundary depends on the trailing $W$-step share history, which is absent from the state observation, making the constraint non-Markovian from the agent's perspective.
\textit{Shared trunk}: the share and effort heads share a common network trunk; a two-trunk architecture would provide stricter causal isolation of the two failure modes.

\paragraph{Broader impact and future directions.}\label{sec:future-directions}
The framework can inform ergonomic interventions extending productive
working life; risks of discriminatory reuse are mitigated by
population-level synthetic agents and no individual-prediction API.
Priority experiments: convex share penalty
($-\lambda \log(s_{\text{dom}} - \alpha)$), non-DP-optimal share
control, SAC comparison, a sweep over initial effort-head bias
values to test whether the escape rate increases monotonically,
replacing the fixed proxy signal with a learned or evolved
internal-signal architecture, to test whether adaptive afferents
restore completion without action-space restriction, and a two-trunk
architecture to provide a stricter causal isolation of the share and
effort failure modes.

\bibliographystyle{plainnat}
\bibliography{refs}

\appendix
\section{Horizon Invariance of the Basin Structure}\label{appendix:horizon-invariance}

This appendix reports the horizon-invariance experiment supporting
prediction P4 (\S\ref{sec:predictions}) and the first-phase
basin-entry claim of \S\ref{sec:optimality:step-zero}. The bricklayer
fixed-share Dyna sweep was run at four training horizons,
$H \in \{13, 20, 30, 49\}$, with 10 seeds per horizon and all other
hyperparameters held fixed (Appendix~\ref{sec:hyperparams}).

\paragraph{Evaluation protocol.}
All policies, regardless of training horizon, are evaluated on the
full $H = 49$ environment. A policy trained at $H = 13$ is therefore
evaluated zero-shot on the career up to age~65; its
$M_{\text{final}}$ is measured at step~48 (age~65). This ensures
all four training horizons produce comparable $M_{\text{final}}$
values. The horizon-invariance of $M_{\text{final}} = 0.523 \pm
0.001$ across all four conditions therefore cannot be attributed to
differences in evaluation length.

\paragraph{Findings.} (i) The reactive-decline basin holds 6--8 of 10
seeds at every $H$ (Figure~\ref{fig:horizon-invariance}a); the
dominance rate does not increase with longer training.
(ii) Dominant-mode effort schedules overlap to within seed noise
across the four horizons (panel b): the policy learned from 13
training years is indistinguishable from the policy learned from 49.
(iii) Dominant-mode $M_{\text{final}}$ is invariant at
$0.523 \pm 0.001$ (panel c); the gap to DP-optimal ($0.271$) does
not shrink with $H$, confirming first-phase basin entry
(\S\ref{sec:optimality:step-zero}).
NBA replication of P4 at $H \in \{7, 10, 15, 20\}$ is reported in
\S\ref{sec:replication}.

\begin{figure}[h]
\centering
\includegraphics[width=\linewidth]{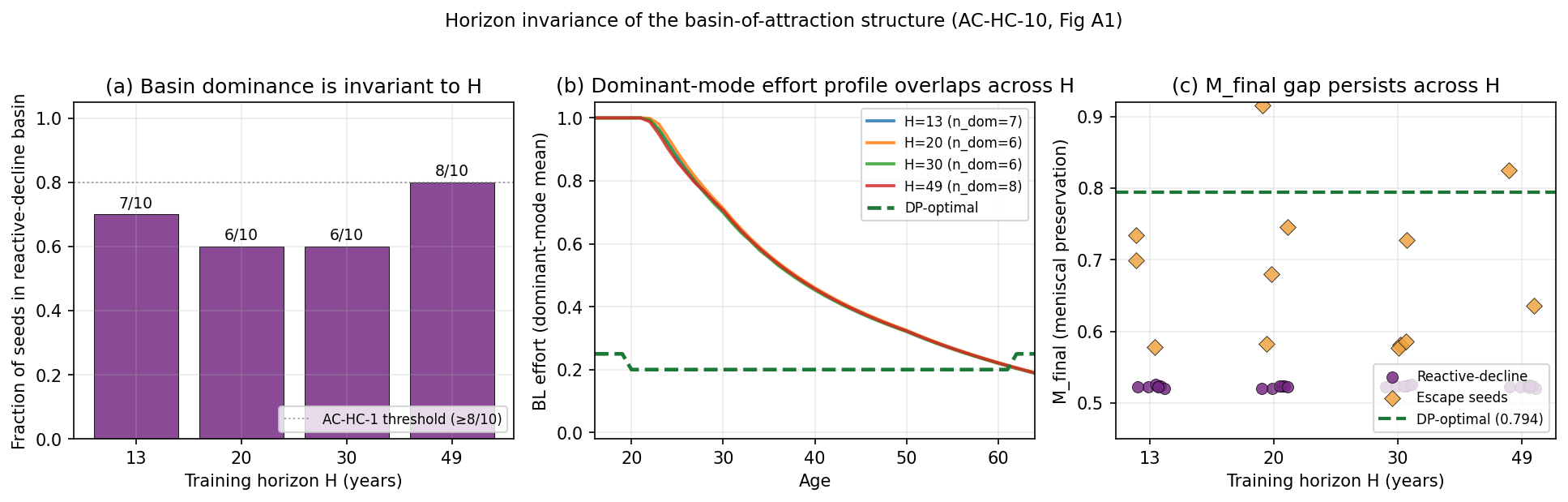}
\caption{Horizon invariance of the basin-of-attraction structure
(bricklayer testbed). All conditions evaluated on the full $H=49$
environment.
(a) Reactive-decline basin holds 6--8 of 10 seeds across training
horizons $H \in \{13, 20, 30, 49\}$.
(b) Dominant-mode effort profiles overlap nearly perfectly across
$H$: the policy learned from 13 years of training rollouts is
indistinguishable in effort schedule from the policy learned from 49
years.
(c) Dominant-mode $M_{\text{final}}$ is invariant across $H$ at
$0.523 \pm 0.001$; the gap to DP-optimal ($0.271$) does not shrink
with longer training. This is direct evidence for first-phase
basin entry: the reactive attractor is determined by the first
policy-gradient update, not by the long-horizon signal that longer
rollouts provide.
Data: cluster~152494 ($H=49$) and cluster~152523 ($H \in \{13,20,30\}$).}
\label{fig:horizon-invariance}
\end{figure}

\section{Step-0 Basin Entry in a Minimal Cumulative-Damage MDP}
\label{sec:synthetic-mdp}

This appendix derives Proposition~\ref{prop:step-zero} on a minimal
MDP that retains only the features essential to the step-0
mechanism. The minimal MDP omits the multi-activity structure,
secondary-channel amplification, and age-dependent capacity of the
testbeds; it therefore speaks to the \emph{direction of the expected
first update} and the $H^*$ boundary, not to the magnitude of
$\Delta M_{\text{final}}$ or the full attractor structure.

\paragraph{Setup.}
Finite-horizon MDP with scalar damage $D_t \in [0,1]$, $D_0 = 0$;
binary effort $e_t \in \{e_L, e_H\}$, $0 < e_L < e_H \leq 1$;
transition $D_{t+1} = \min(1, D_t + \kappa e_t)$; reward
$r(D_t, e_t) = e_t^{\beta}(1-D_t)^2$ for $\beta \in (0,1]$;
return $G(\tau) = \sum_{t=0}^{H-1} r(D_t, e_t)$.
This satisfies Assumption~\ref{asm:cds} and is the skeleton on
which the step-0 mechanism operates.

\begin{lemma}[Lipschitz bound on $V^\pi$]
\label{lem:lipschitz}
Let $\bar{e}_\beta = \tfrac{1}{2}(e_L^\beta + e_H^\beta)$ be the
expected per-step effort reward under the uniform policy at $D>0$.
Under the deterministic transition $D_{t+1} = D_t + \kappa e_t$
(prior to clipping at $1$), the continuation value satisfies
\[
  \bigl|V^\pi(\kappa e_H) - V^\pi(\kappa e_L)\bigr|
  \;\leq\; 2(H-1)\,\bar{e}_\beta\,\kappa(e_H - e_L).
\]
\emph{Proof.} The reward $r(D,e) = e^\beta(1-D)^2$ satisfies
$|\partial r/\partial D| = 2e^\beta(1-D) \leq 2e^\beta$ on $[0,1]$.
Under the uniform policy for $t \geq 1$, damage evolves
deterministically, so $\partial D_t/\partial D_1 = 1$ for all
$t \geq 1$. The continuation value $V^\pi(\kappa e)$ covers
steps $t = 1, \ldots, H-1$ (i.e., $H-1$ steps). By the chain rule:
\[
  \left|\frac{dV^\pi}{dD_1}\right|
  \leq \sum_{t=1}^{H-1} \mathbb{E}\bigl[2e_t^\beta\bigr]
  \leq 2(H-1)\,\bar{e}_\beta.
\]
Applying this to the gap $V^\pi(\kappa e_H) - V^\pi(\kappa e_L)$
via the mean-value theorem in $D_1$ gives the stated bound. $\square$
\end{lemma}

\paragraph{Policy and first-update gradient.}
Parameterise the $D=0$ action distribution by $\theta \in \mathbb{R}$:
$\pi_\theta(e_H \mid D=0) = \sigma(\theta)$, initialised at
$\theta_0 = 0$ (uniform). Action selection at $D > 0$ is held at
uniform $\Pr(e_H) = 1/2$; we analyse only the first update on
$\theta$. The REINFORCE gradient is
$\nabla_\theta J(\theta) = \mathbb{E}_\pi\bigl[\nabla_\theta \log \pi_\theta(a_0 \mid D_0)\, G(\tau)\bigr]$.
At $\theta_0 = 0$, $\sigma(\theta_0) = 1/2$ and
$\nabla_\theta \log \pi_{\theta_0}(e_H \mid 0) = +1/2$,
$\nabla_\theta \log \pi_{\theta_0}(e_L \mid 0) = -1/2$. Splitting
the expectation by the first action and identifying the inner
expectations as $Q^\pi(0, e_H)$ and $Q^\pi(0, e_L)$ yields
\begin{equation}
\nabla_\theta J(\theta)\big|_{\theta_0}
= \tfrac{1}{4}\bigl[Q^\pi(0, e_H) - Q^\pi(0, e_L)\bigr].
\label{eq:step-zero-gradient}
\end{equation}
The sign of the \emph{expected} first update equals the sign of the
state-action-value gap at $D=0$. Using the Bellman decomposition
$Q^\pi(0,e) = r(0,e) + V^\pi(\kappa e)$ and Lemma~\ref{lem:lipschitz}:
\begin{equation}
Q^\pi(0,e_H) - Q^\pi(0,e_L)
\geq (e_H^\beta - e_L^\beta)
     - 2(H-1)\,\kappa\,\bar{e}_\beta\,(e_H - e_L).
\label{eq:commitment-condition-exact}
\end{equation}
Equivalently, writing $e_H^\beta - e_L^\beta = \Delta_\beta$ for
brevity:
\begin{equation}
Q^\pi(0,e_H) - Q^\pi(0,e_L)
\geq \Delta_\beta - 2(H-1)\,\kappa\,\bar{e}_\beta\,(e_H - e_L).
\label{eq:commitment-condition}
\end{equation}

\begin{proposition}[Step-0 commitment]
\label{prop:step-zero}
If
\[
  e_H^\beta - e_L^\beta \;>\; 2(H-1)\,\kappa\,\bar{e}_\beta\,(e_H - e_L),
\]
then the \emph{expected} first policy-gradient update is positive,
$\mathbb{E}[\nabla_\theta J(\theta)|_{\theta_0}] > 0$: the expected
update shifts $\theta$ toward $e_H$ at $D = 0$. Because the
condition depends only on the MDP parameters $(\beta, \kappa, H,
e_L, e_H)$ and not on random seed, the \emph{expected} gradient
favours $e_H$ for every seed. Individual stochastic gradient
estimates may differ in sign (consistent with the two escaping seeds
in \S\ref{sec:optimality:step-zero}).
\end{proposition}

\paragraph{Critical horizon.}
Define
\[
  H^* = 1 + \frac{e_H^\beta - e_L^\beta}{2\,\kappa\,\bar{e}_\beta\,(e_H - e_L)}.
\]
Condition~\eqref{eq:commitment-condition} holds for all $H \leq H^*$.

\paragraph{Scope and self-reinforcement.}
The proposition establishes the direction of the expected first
update only. It does not prove convergence to the $e_H = 1$
attractor. Self-reinforcement bridges the direction of the first update to
empirical saturation: (i)~after the first expected update
$\mathbb{E}[\theta_1] > 0$, rollouts from $D_0 = 0$ increasingly
sample $e_H$, concentrating subsequent gradient signal on
$e_H$-prefixed trajectories; (ii)~early-rollout damage is small,
so reward suppression is dominated by the $e_H^\beta - e_L^\beta$
reward gap for the first $\sim 1/\kappa$ steps; (iii)~subsequent
updates receive a stronger signal of the same sign, driving $\theta$
monotonically until $\pi(e_H \mid D_0 = 0) \to 1$. This is
consistent with empirical saturation within the first 1\% of
training for all 8 reactive-cluster seeds
(\S\ref{sec:optimality:step-zero}).
The proposition does not cover the multi-activity,
continuous-effort testbeds. It is a motivating formal result, not a
direct theorem about those environments.

\paragraph{Bricklayer $H^*$ calculation.}
In the bricklayer, under greedy full effort ($e_H = 1$) on the
dominant activity, the per-step damage rate is
$\kappa \approx \text{damage\_scale} \times \text{load} \approx
0.083 \times 0.9 \approx 0.075$. Using $\beta = 0.6$,
$e_L \approx 0.05$, $e_H = 1.0$:
\[
  e_H^\beta - e_L^\beta = 1.0 - 0.05^{0.6} \approx 1.0 - 0.178 = 0.822,
  \quad
  \bar{e}_\beta = \tfrac{1}{2}(0.178 + 1.0) = 0.589,
  \quad
  e_H - e_L = 0.95.
\]
\[
  H^* = 1 + \frac{0.822}{2 \times 0.075 \times 0.589 \times 0.95}
      = 1 + \frac{0.822}{0.0840}
      \approx 10.8.
\]
The analytical guarantee therefore covers $H \leq 10$. All four
bricklayer training horizons ($H \in \{13, 20, 30, 49\}$) lie above
$H^*$, yet empirical basin entry occurs at all four (8/10 seeds each,
$M_{\text{final}} = 0.523 \pm 0.001$). This is consistent with the
condition being sufficient but not necessary: the proposition predicts
commitment when $H \leq H^*$, while the empirical results show that
the basin remains dominant well beyond the analytical guarantee, as in
the NBA environment.

\emph{NBA.} The NBA environment has two distinct damage channels:
the primary cartilage channel ($\kappa = \delta = 0.055$) and the
secondary meniscal channel ($\kappa = \mu_{\text{eff}} = 0.15$).
Using $\beta = 0.6$, $e_L = 0.05$, $e_H = 1.0$ (consistent with
the assumption $0 < e_L < e_H$ and the NBA minimum-effort floor):
\[
  e_H^\beta - e_L^\beta = 0.822,\quad
  \bar{e}_\beta = 0.589,\quad
  e_H - e_L = 0.95.
\]
\[
  H^*_{\text{primary}}
    = 1 + \frac{0.822}{2 \times 0.055 \times 0.589 \times 0.95}
    \approx 1 + \frac{0.822}{0.0615}
    \approx 14.4,
\]
\[
  H^*_{\text{secondary}}
    = 1 + \frac{0.822}{2 \times 0.15 \times 0.589 \times 0.95}
    \approx 1 + \frac{0.822}{0.168}
    \approx 5.9.
\]
The overall guarantee floor is $H^* \approx 6$ (secondary channel
binding) and ceiling $H^* \approx 14$ (primary channel), giving a
range of $H^* \in [6, 14]$ depending on which channel is the binding
constraint. The horizon $H = 20$ lies above both bounds. Empirical
basin entry at 4/10 seeds is consistent with the condition being
sufficient but not necessary. The horizon-sweep
outlier at $H = 15$ lies just above $H^*_{\text{primary}} \approx
14$: it is the first integer horizon outside the formal commitment
region under the primary channel, making the $H = 15$ exception the
sharpest available confirmation of the $H^*$ boundary on the
discrete horizon grid.

\section{Biological Parameter Provenance}\label{sec:bio-params}

The bricklayer environment models a 49-year career with seven
activities parameterised by energy cost, hazard coefficient, and
performance coefficient (Table~\ref{tab:activities}). Per-step
cumulative damage evolves as
\begin{equation}
D_{t+1} = \max\!\Bigl(0,\;\min\!\bigl(1,\;
  D_t
  + \text{damage\_scale} \cdot \text{load}_t
    \cdot b(\mathrm{BMI}_t) \cdot m(M_t)^{1.3}
  - r(\text{age}_t)\bigr)\Bigr),
\label{eq:damage-update}
\end{equation}
where $\text{load}_t = 0.40 \cdot \text{stress}_t
+ 0.35 \cdot \text{strain}_t + 0.25 \cdot \text{shear}_t$
is the weighted load. For each activity $i$, the load components are
derived from the single hazard scalar $h_i$ as
$\text{stress}_i = 0.45\,h_i$, $\text{strain}_i = 0.35\,h_i$,
$\text{shear}_i = 0.20\,h_i$ (proportions calibrated to knee-joint
biomechanics \citealt{jensen2008}), so
$\text{load}_t = \bigl(0.40\times0.45 + 0.35\times0.35 +
0.25\times0.20\bigr)\sum_i s_{i,t}\,e_{i,t}\,h_i
= 0.355\,\sum_i s_{i,t}\,e_{i,t}\,h_i$.
$\text{damage\_scale} = 0.083$ calibrated to
a no-regulation end-of-career $D \approx 0.75$
\citep{coggon2000,jensen2008}, $b(\mathrm{BMI}) = 1 + 0.07\,
\max(0, \mathrm{BMI} - 22)$ is a BMI multiplier, and
$r(\text{age}) = 0.015 \cdot \max(0, 1 - (\text{age}-30)/50)$
is a mild age-dependent recovery term (clamped to maintain
$D \geq 0$).

Meniscal degradation evolves as
\[
M_{t+1} = M_t
  - \text{base\_rate} \cdot \text{shear}_t \cdot a(M_t)
    \cdot f(\text{age}_t),
\]
with $\text{base\_rate} = 0.075$ derived from a target end-of-career
$M \approx 0.65$ under 70\% block\_laying share at 70\% effort
\citep{rytter2009}. The self-amplification factor is
\[
  a(M) = \begin{cases}
    1 + 3(0.6 - M) & \text{if } M < 0.6, \\
    1              & \text{otherwise,}
  \end{cases}
\]
and the age-onset factor is
\[
  f(\text{age}) = \begin{cases}
    1 + 0.5\,(\text{age} - 45)/20 & \text{if age} \geq 45, \\
    1                              & \text{otherwise.}
  \end{cases}
\]
Baratz amplification $m(M) = 1/(0.45 + 0.55 M)$ matches
the 55\% contact-area reduction in \citet{baratz1986}.

The proxy signal is
$S_t = 1 + \alpha_s D_t^{\gamma_s}$ with $\alpha_s = 2.5$,
$\gamma_s = 1.8$ (sensitisation parameters fit to the
pressure-pain threshold data of \citealt{arendtnielsen2010}).
The reward function is
\[
  r_t = h_{\text{occ}}(D_t, S_t)\;\sum_i s_{i,t}\,e_{i,t}^{\,\beta}\,
        \mathrm{perf}_i / \mathrm{perf}_{\max},
\]
where $h_{\text{occ}}(D, S) = \max(0, 1 - \mathbf{1}_{D > D_{\text{clin}}}
\cdot \nu (D - D_{\text{clin}}) \cdot S)$ is the HMS capacity-damage
feedback with $D_{\text{clin}} = 0.30$ and $\nu = 1.5$,
and $\beta = 0.6$.

\begin{table}[h]
\centering
\small
\caption{Bricklayer activity parameters (energy/hazard scale 0--100;
performance scale 0--110, where values above 100 encode above-average
task output).}
\label{tab:activities}
\begin{tabular}{lrrrl}
\toprule
Activity & Energy & Hazard & Perf & Anchor \\
\midrule
block\_laying       & 85 & 90 & 105 & \citet{coggon2000,jensen2008} \\
scaffold\_work      & 80 & 55 &  85 & \citet{vandermolen2005} \\
mortar\_mixing      & 60 & 50 &  60 & \citet{ainsworth2011} \\
cutting\_grinding   & 65 & 35 &  65 & \citet{ainsworth2011} \\
pointing\_finishing & 35 & 25 &  62 & \citet{ainsworth2011} \\
light\_repair       & 20 & 10 &  32 & \citet{ainsworth2011} \\
coordination        & 15 &  5 &  45 & \citet{ainsworth2011} \\
\bottomrule
\end{tabular}
\end{table}

\section{Hyperparameters and Implementation}\label{sec:hyperparams}

\paragraph{PPO.} Discount $\gamma = 1.0$ (undiscounted episodic
return); rbar exponential moving-average baseline
$\bar{r} \leftarrow (1-\eta)\bar{r} + \eta r_t$, $\eta = 0.01$
\citep{mahadevan1996}; training steps $10^6$ (primary),
$H \cdot 2 \times 10^4$ (horizon sweep); clip $\epsilon = 0.2$;
entropy coefficient $c_{\text{ent}} = 0.05$;
value coefficient $c_v = 0.5$; learning rate $3 \times 10^{-4}$
(Adam); minibatch 64; epochs 4; rollout 2048.

\paragraph{Dyna variants \citep{sutton1990dyna}.}
The Dyna agent interleaves real-environment rollouts with
planning steps using the ground-truth simulator as the world model.
The \texttt{fixed-share} condition freezes activity shares at the
age-stratified DP-optimal allocation and learns only per-activity
effort. The \texttt{unrestricted} condition learns both shares and
efforts. Both use the \texttt{no\_exit} flag (role\_exit,
capacity\_exit, and age\_limit checks disabled during training);
enforcement is always active at evaluation.

\paragraph{Policy network.} Shared trunk: two FC layers (128 units,
ReLU); Dirichlet share head; Normal effort head (concatenates trunk
output and sampled shares; clamped to $[0,1]$).

\paragraph{Dynamic programming.} Backward induction on
$(D, M, \text{age})$ grid, resolution $(0.01, 0.01, 1)$, effort
grid $\{0.20, 0.25, \ldots, 0.80\}$ at the DP-optimal share
allocation. The age-stratified DP-optimal shares used in the
\texttt{fixed\_share} condition are read directly from the
argmax of this backward-induction pass. Compute time: under 2 hr on a single CPU.

\paragraph{Initialization intervention.} Effort head final-layer
bias set to $-1.386$ so $\sigma(-1.386) \approx 0.20$; all other
hyperparameters and seeds $\{0,\ldots,9\}$ identical to the
fixed-share sweep. Pre-registered binding thresholds frozen
2026-04-20 at \texttt{prompts/e1\_prereg.md}.

\paragraph{Evaluation.} 100 independent rollouts per seed; 95\% CIs
from empirical distribution across seeds.

\begin{table}[h]
\caption{Initialization-bias intervention (10 seeds). Reactive-basin
membership requires initial-state effort $> 0.80$, mid-horizon
effort $> 0.40$, late-horizon effort $< 0.30$, and monotone decline
(tolerance $0.02$). Pre-registered classifier frozen 2026-04-20.
The four escaping seeds ($s \in \{0,4,7,9\}$) split into three
distinct attractors; only seeds~4 and~7 reach the DP-competitive
region.}
\label{tab:e1-seed-table}
\centering\small
\begin{tabular}{cccccr}
\toprule
Seed & Effort@16 & Effort@30 & Effort@60 & Monotone & $M_{\text{final}}$ \\
\midrule
0 & 1.000 & 0.685 & 0.223 & no  & 0.538 \\
1 & 1.000 & 0.697 & 0.220 & yes & 0.524 \\
2 & 1.000 & 0.725 & 0.223 & yes & 0.523 \\
3 & 0.997 & 0.708 & 0.217 & yes & 0.525 \\
4 & 0.209 & 0.213 & 0.235 & no  & \textbf{0.793} \\
5 & 1.000 & 0.703 & 0.221 & yes & 0.524 \\
6 & 1.000 & 0.753 & 0.229 & yes & 0.521 \\
7 & 0.143 & 0.174 & 0.145 & no  & \textbf{0.846} \\
8 & 1.000 & 0.681 & 0.223 & yes & 0.523 \\
9 & 1.000 & 0.841 & 0.000 & no  & \textbf{0.777} \\
\bottomrule
\end{tabular}
\end{table}

\section{Full-action-space DP: CMA-ES partial relaxation}\label{sec:fullaction-dp}

We use CMA-ES (population 50) to optimise a 98-dimensional vector
parameterising per-year dominant-activity share $s \in [0.15, 0.80]$
and effort $e \in [0.05, 1.00]$ for each of the 49 career years
(non-dominant shares uniform; non-dominant efforts fixed at $0.40$).
Over 300 generations (15{,}000 evaluations), CMA-ES converges to a
solution matching fixed-share DP to three significant figures in
reward, $M_{\text{final}}$, and $D_{\text{final}}$: within the
partial relaxation tested, the fixed-share restriction is not
binding. A true full-action DP (14-dimensional per year) remains
to be computed; if it yields $M_{\text{final}} > 0.794$, the
$0.271$ gap widens; it is already a lower bound on the shortfall.

\section{Role-violation sensitivity}\label{sec:role-sensitivity}

A joint sweep over $W \in \{3,5,7\}$ and $\alpha \in \{0.10,0.15,0.20\}$
(5 seeds per cell, 90 runs) finds the completion pattern fully
robust: PPO-real achieves completion rate \textbf{exactly 0} across
all 45 runs; fixed-share Dyna achieves completion rate \textbf{exactly 1}
across all 45 runs at every $(W, \alpha)$ cell. The completion
failure of PPO-real is not an artefact of the specific threshold.

\section{NBA Replication: Parameter Provenance}\label{sec:nba-params}

Six activities with post\_play as dominant activity (unique reward
maximiser: $\mathrm{perf} = 110$; unique hazard maximiser:
$\mathrm{hazard} = 90$). Normalised-hazard load model:
$\mathrm{load} = \sum_i s_i e_i h_i / h_{\max}$, $h_{\max} = 90$.
Damage:
$D_{t+1} = \max\bigl(0, \min\bigl(1,\;D_t + \delta\,\mathrm{load}_t\,m(M_t)^{1.3}
- \rho(\mathrm{age}_t)\bigr)\bigr)$,
$\delta = 0.055$. Meniscal: $\mu_{\text{eff}} = 0.15$,
amplification and age-onset factors use the same functional forms as
Appendix~\ref{sec:bio-params} with parameters re-calibrated to the
NBA hazard array. Role-viability: trailing 3-season
post\_play share $< \alpha = 0.12$ triggers release. DP grid
resolution $\Delta D = 0.02$, $\Delta M = 0.053$; solve time
$\sim 12$\,s. Capacity feedback, HMS pathway, and Baratz
amplification use the same functional forms as the bricklayer
(scaled by the NBA hazard array). Pre-registration:
\texttt{docs/prereg\_nba.md}.

\begin{table}[h]
\small\centering
\caption{NBA power-forward activities.}
\label{tab:nba-activities}
\begin{tabular}{lrrrl}
\toprule
Activity & Energy & Hazard & Perf & Description \\
\midrule
post\_play             & 90 & 90 & 110 & rebounding, post-ups, rim protection \\
perimeter\_play        & 70 & 45 &  75 & spot-up shooting, transition \\
full\_practice         & 75 & 55 &  60 & full-contact practice \\
skill\_training        & 40 & 15 &  45 & individual shooting, footwork \\
strength\_conditioning & 55 & 25 &  35 & weight room, plyometrics \\
rehab\_rest            & 10 &  3 &  10 & active recovery \\
\bottomrule
\end{tabular}
\end{table}

\section*{NeurIPS Paper Checklist}

\begin{enumerate}[leftmargin=*,label=\textbf{\arabic*.}]

\item \textbf{Claims}

Question: Do the main claims made in the abstract and introduction
accurately reflect the paper's contributions and scope?

Answer: [Yes]

Justification: The abstract's four claims (decomposition
(\S\ref{sec:decomposition}), completion gap (\S\ref{sec:completion}),
optimality gap (\S\ref{sec:optimality}), and NBA replication
(\S\ref{sec:replication})) are each supported by experiments in the
cited sections. The ``harmful horizon access'' claim is explicitly
scoped to PPO with a linear soft penalty in both the abstract and the
contributions list.

\item \textbf{Limitations}

Question: Does the paper discuss the limitations of the work
performed by the authors?

Answer: [Yes]

Justification: An explicit Limitations paragraph in
\S\ref{sec:discussion} covers: same-engine replication scope
(parameter-regime generality only), Proposition--testbed scope gap
(binary-action MDP vs.\ continuous multi-activity testbeds),
algorithm/penalty specificity of P2 (PPO with linear soft penalty
only), pre-registered nuanced result (initialization intervention),
ground-truth model assumption (exact transition function in Dyna),
partial observability of the trailing-window role-viability
constraint, and shared-trunk coupling between the share and effort
heads.

\item \textbf{Theory assumptions and proofs}

Question: For each theoretical result, does the paper provide the
full set of assumptions and a complete (and correct) proof?

Answer: [Yes]

Justification: Proposition~\ref{prop:step-zero} is stated under
Assumption~\ref{asm:cds} with the additional assumption
$0 < e_L < e_H \leq 1$, and proved in
Appendix~\ref{sec:synthetic-mdp}. The supporting
Lemma~\ref{lem:lipschitz} (Lipschitz bound on $V^\pi$) is stated and
proved in full. Scope (binary-action minimal MDP, direction of
expected first update only, not convergence) is demarcated in both
the main text and the appendix.

\item \textbf{Experimental result reproducibility}

Question: Does the paper fully disclose all the information needed
to reproduce the main experimental results of the paper to the
extent that it affects the main claims and/or conclusions of the
paper (regardless of whether the code and data are provided or not)?

Answer: [Yes]

Justification: Hyperparameters are in Appendix~\ref{sec:hyperparams};
biological parameters including explicit formulas for the
self-amplification factor $a(M)$, age-onset factor $f(\text{age})$,
proxy signal $S_t$, reward function $r_t$, and load decomposition are
in Appendix~\ref{sec:bio-params}; NBA parameters are in
Appendix~\ref{sec:nba-params}. Configuration files are released in
the supplementary repository.

\item \textbf{Open access to data and code}

Question: Does the paper provide open access to the data and code,
with sufficient instructions to faithfully reproduce the main
experimental results, as described in supplemental material?

Answer: [Yes]

Justification: Code, configuration files, and per-seed
\texttt{summary.json} result files are released alongside the paper
in the supplementary repository, including the \texttt{EnergyCareerEnv}
simulator, training scripts, DP solver, and analysis pipelines.

\item \textbf{Experimental setting/details}

Question: Does the paper specify all the training and test details
(e.g., data splits, hyperparameters, how they were chosen, type of
optimizer) necessary to understand the results?

Answer: [Yes]

Justification: Full training details are in
Appendix~\ref{sec:hyperparams}: optimizer (Adam,
$\text{lr}=3\times10^{-4}$), PPO clip ($\epsilon=0.2$), entropy and
value coefficients, rollout length, minibatch size, and training
steps. The role-viability rule and \texttt{no\_exit} flag semantics
are defined in \S\ref{sec:completion}. The fixed-share PPO-real
ablation protocol is reported in \S\ref{sec:completion}.

\item \textbf{Experiment statistical significance}

Question: Does the paper report error bars suitably and correctly
defined or other appropriate information about the statistical
significance of the experiments?

Answer: [Yes]

Justification: All main results include appropriate statistics:
standard deviations on exit age; 95\% bootstrap CIs on
$\Delta M_{\text{final}}$; Welch's $t$-test with explicit degrees of
freedom and Mann-Whitney $U$ for exit-age comparisons; 95\%
Clopper-Pearson CIs on basin-dominance rates; Fisher exact test for
the cross-environment dominance-rate comparison. The source of
variability (random seed) is stated throughout.

\item \textbf{Experiments compute resources}

Question: For each experiment, does the paper provide sufficient
information on the computer resources (type of compute workers,
memory, time of execution) needed to reproduce the experiments?

Answer: [Yes]

Justification: Appendix~\ref{sec:hyperparams} reports: each PPO/Dyna
run $\sim$13 min on a single CPU node; DP solve under 2 hr on a
single CPU; role-sensitivity sweep $\sim$20 HPC node-hours; NBA
sweeps $\sim$15 node-hours total.

\item \textbf{Code of ethics}

Question: Does the research conducted in the paper conform, in every
respect, with the NeurIPS Code of Ethics?

Answer: [Yes]

Justification: No human subjects, crowdsourcing, or proprietary data
are involved. The environments are calibrated exclusively to published
epidemiological aggregates. No individual-prediction API is released.

\item \textbf{Broader impacts}

Question: Does the paper discuss both potential positive societal
impacts and negative societal impacts of the work performed?

Answer: [Yes]

Justification: \S\ref{sec:discussion} discusses beneficial uses
(ergonomic interventions extending productive working life, public
health policy) and risks (age-discriminatory hiring or insurance
pricing), with three explicit mitigations: population-level synthetic
agents only, aggregate epidemiological calibration, and no
individual-prediction API in released code.

\item \textbf{Safeguards}

Question: Does the paper describe safeguards that have been put in
place for responsible release of data or models that have a high
risk for misuse?

Answer: [N/A]

Justification: The paper releases a simulation framework and analysis
code only, not a pre-trained model or scraped dataset. The simulator
produces synthetic population-level trajectories and poses no direct
misuse risk.

\item \textbf{Licenses for existing assets}

Question: Are the creators or original owners of assets (e.g., code,
data, models), used in the paper, properly credited and are the
license and terms of use explicitly mentioned and properly respected?

Answer: [Yes]

Justification: All biological and sports-medicine literature is cited
at point of use and in the parameter appendices. All code
dependencies (PyTorch, NumPy, etc.) are open-source. No proprietary
datasets or models are used.

\item \textbf{New assets}

Question: Are new assets introduced in the paper well documented and
is the documentation provided alongside the assets?

Answer: [Yes]

Justification: The supplementary repository contains the
\texttt{EnergyCareerEnv} simulator, training and DP solver scripts,
analysis pipelines, and configuration files for both testbeds, with
a README describing installation and reproduction steps.

\item \textbf{Crowdsourcing and research with human subjects}

Question: For crowdsourcing experiments and research with human
subjects, does the paper include the full text of instructions given
to participants and screenshots, if applicable, as well as details
about compensation (if any)?

Answer: [N/A]

Justification: The paper involves no crowdsourcing and no research
with human subjects. All environments are synthetic simulators
calibrated to published epidemiological data.

\item \textbf{Institutional review board (IRB) approvals or
equivalent for research with human subjects}

Question: Does the paper describe potential risks incurred by study
participants, whether such risks were disclosed to the subjects, and
whether Institutional Review Board (IRB) approvals (or an equivalent
approval/review based on the requirements of your country or
institution) were obtained?

Answer: [N/A]

Justification: No human subjects are involved; IRB approval is not
required.

\item \textbf{Declaration of LLM usage}

Question: Does the paper describe the usage of LLMs if it is an
important, original, or non-standard component of the core methods
in this research?

Answer: [N/A]

Justification: LLMs were used only for writing and editing assistance
and are not a component of the core methodology, scientific
contributions, or experimental results.

\end{enumerate}

\end{document}